\documentclass{acm_proc_article-sp}
\newcommand\car{{\mathbf 1}}

\newcommand\R{{\mathbf R}}
\renewcommand\P{{\mathbf P}}
\renewcommand\d{\text{ d}}
\newcommand\esp[2]{{\mathbf E}_{#1}\left[#2\right]}
\newcommand\Esp[1]{{\mathbf E}\left[#1\right]}
\newcommand\Var[1]{{\mathbf V}\left[#1\right]}
\newcommand\Nmax{{N_{\text{max}}}}
\newcommand\Ntot{{N_{\text{tot}}}}
\newcommand\Navail{{N_{\text{avail}}}}
\newcommand\SNR{{\text{SNR}}}
\newcommand\dom{{\text{Dom }}}
\newcommand\lip{{\text{Lip}}}
\newcommand\dref{d_{\text{ref}}}
\newtheorem{theorem}{Theorem}
\begin{document}

\title{Robust methods for LTE and WiMAX dimensioning}

%
%
%
%
%

\numberofauthors{4} 
%
\author{
  \alignauthor
  L. Decreusefond\\
  \affaddr{Institut Telecom}\\
  \affaddr{Telecom ParisTech}\\
  \affaddr{Paris, France}
  \alignauthor
  E. Ferraz\\
  \affaddr{Institut Telecom}\\
  \affaddr{Telecom ParisTech}\\
  \affaddr{Paris, France}
  \and \alignauthor
  P. Martins\\
  \affaddr{Institut Telecom}\\
  \affaddr{Telecom ParisTech}\\
  \affaddr{Paris, France}
  \alignauthor
  T.T. Vu\\
  \affaddr{Institut Telecom}\\
  \affaddr{Telecom ParisTech}\\
  \affaddr{Paris, France}
} \date{\today}

\maketitle{}
\begin{abstract}
  This paper proposes an analytic model for dimensioning OFDMA based
  networks like WiMAX and LTE systems. In such a system, users require
  a number of subchannels which depends on their \SNR, hence of their
  position and the shadowing they experience. The system is overloaded
  when the number of required subchannels is greater than the number
  of available subchannels. We give an exact though not closed
  expression of the loss probability and then give an algorithmic
  method to derive the number of subchannels which guarantees a loss
  probability less than a given threshold. We show that Gaussian
  approximation lead to optimistic values and are thus unusable. We
  then introduce Edgeworth expansions with error bounds and show that
  by choosing the right order of the expansion, one can have an
  approximate dimensioning value easy to compute but with guaranteed
  performance. As the values obtained are highly dependent from the
  parameters of the system, which turned to be rather undetermined, we
  provide a procedure based on concentration inequality for Poisson
  functionals, which yields to conservative dimensioning.  This paper
  relies on recent results on concentration inequalities and establish
  new results on Edgeworth expansions.
\end{abstract}


\section{Introduction}
\label{intro}
Future wireless systems will widely rely on OFDMA (Orthogonal
Frequency Division Multiple Access) multiple access technique. OFDMA
can satisfy end user's demands in terms of throughput. It also
fulfills operator's requirements in terms of capacity for high data
rate services. Systems such as 802.16e and 3G-LTE (Third Generation
Long Term Evolution) already use OFDMA on the downlink. Dimensioning
of OFDMA systems is then of the utmost importance for wireless
telecommunications industry.

OFDM (Orthogonal Frequency Division Multiplex) is a multi carrier
technique especially designed for high data rate services. It divides
the spectrum in a large number of frequency bands called (orthogonal)
subcarriers that overlap partially in order to reduce spectrum
occupation.  Each subcarrier has a small bandwidth compared to the
coherence bandwidth of the channel in order to mitigate frequency
selective fading. User data is then transmitted in parallel on each
sub carrier.  In OFDM systems, all available subcarriers are affected
to one user at a given time for transmission. OFDMA extends OFDM by
making it possible to share dynamically the available subcarriers
between different users. In that sense, it can then be seen as
multiple access technique that both combines FDMA and TDMA
features. OFDMA can also be possibly combined with multiple antenna
(MIMO) technology to improve either quality or capacity of systems.

\begin{figure}[!ht]
  \centering
  \includegraphics[width=\columnwidth]{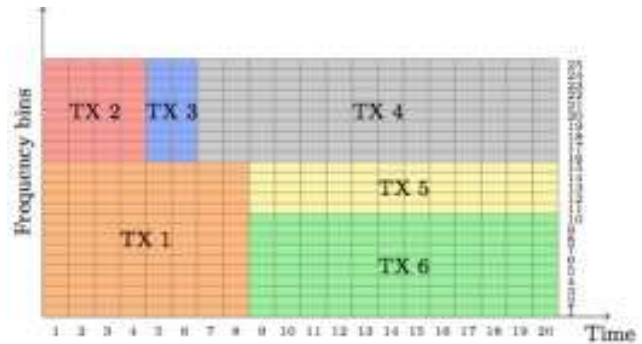}
  \caption{OFDMA principle : subcarriers are allocated according to
    the required transmission rate}
  \label{fig_dfmv_valuetools:ofdma}
\end{figure}

In practical systems, such as WiMAX or 3G-LTE, subcarriers are not
allocated individually for implementation reasons mainly inherent to
the scheduler design and physical layer signaling. Several subcarriers
are then grouped in subchannels according to different strategies
specific to each system. In OFDMA systems, the unit of resource
allocation is mainly the subchannels. The number of subchannels
required by a user depends on his channel's quality and the required
bit rate. If the number of demanded subchannels by all users in the
cell is greater than the available number of subchannel, the system is
overloaded and suffer packet losses. The questions addressed here can
then be stated as follows: how many subchannels must be assigned to a
BS to ensure a small overloading probability~? Given the number of
available subchannels, what is the maximum load, in terms of mean
number of customers per unit of surface, that can be tolerated~? Both
questions rely on accurate estimations of the loss probability.

The objectives of this paper are twofold: First, construct and analyze
a general performance model for an isolated cell equipped with an
OFDMA system as described above. We allows several classes of
customers distinguished by their transmission rate and we take into
account path-loss with shadowing.  We then show that for a Poissonian
configuration of users in the cell, the required number subchannels
follows a compound Poisson distribution.  The second objective is to
compare different numerical methods to solve the dimensioning
problem. In fact, there exists an algorithmic approach which gives the
exact result potentially with huge memory consumption. On the other
hand, we use and even extend some recent results on functional
inequalities for Poisson processes to derive some approximations
formulas which turn to be rather effective at a very low cost. When it
comes to evaluate the performance of a network, the quality of such a
work may be judged according to several criteria. First and foremost,
the exactness is the most used criterion: it means that given the
exact values of the parameters, the real system, the performances of
which may be estimated by simulation, behaves as close as possible to
the computed behavior. The sources of errors are of three kinds: The
mathematical model may be too rough to take into account important
phenomena which alter the performances of the system, this is known as
the epistemic risk. Another source may be in the mathematical
resolution of the model where we may be forced to use approximate
algorithms to find some numerical values. The third source lies in the
lack of precision in the determination of the parameters
characterizing the system: They may be hard, if not impossible, to
measure with the desired accuracy.  It is thus our point of view that
exactness of performance analysis is not all the matter of the
problem, we must also be able to provide confidence intervals and
robust analysis. That is why, we insist on error bounds in our
approximations.

Resources allocation on OFDMA systems have been extensively studied
over the last decade, often with joint power and subcarriers
allocation, see for instance
\cite{LowcomplexityresourceallocationwithopportunisticfeedbackoverdownlinkOFDMAnetworks,DownlinkschedulingandresourceallocationforOFDMsystems,OptimalResourceAllocationforOFDMADownlinkSystems,AdaptiveresourceallocationinmultiuserOFDMsystemswithproportionalrateconstraints}. The
problem of OFDMA planning and dimensioning have been more recently
under investigation. In
\cite{DimensioningofOFDMOFDMABasedCellularNetworksUsingExponentialEffectiveSINR},
the authors propose a dimensioning of OFDMA systems focusing on link
outage but not on the other parameters of the systems. In
\cite{DimensioningMethodologyforOFDMANetworksMaqbool}, the authors
give a general methodology for the dimensioning of OFDMA systems,
which mixes a simulation based determination of the distribution of
the signal-to-interference-plus-noise ratio (SINR) and a Markov chain
analysis of the traffic. In
\cite{DimensioningOfTheDownlinkInOFDMACellularNetworksViaAnErlangLossModel,Karray2010Analytical},
the authors propose a dimensioning method for OFDMA systems using
Erlang's loss model and Kaufman-Roberts recursion algorithm. In
\cite{FadingEffectOnTheDynamicPerformanceEvaluationOfOFDMACellularNetworks},
the authors study the effect of Rayleigh fading on the performance of
OFDMA networks.

The article is organized as follows. In Section
\ref{sec:system-model}, we describe the system model and set up the
problem. In Section \ref{sec:loss-probability}, we examine four
methods to derive an exact, approximate or robust value of the number
of subchannels necessary to ensure a given loss probability. In
Section \ref{sec:appl-ofdma-lte}, we apply these formulas to the
particular situation of OFDMA systems. A new bound for the Edgeworth
expansion is in Section \ref{sec:edgeworth-expansion} and Section
\ref{sec:conc-ineq} contains a new proof of the concentration
inequality established for instance in
\cite{Wu:2000lr}.

\section{System Model}
\label{sec:system-model}


In practical systems, such as WiMAX or 3G-LTE, resource allocation
algorithms work at subchannel level. The subcarriers are grouped into
subchannels that the system allocates to different users according to
their throughput demand and mobility pattern. For example, in WiMAX,
there are three modes available for building subchannels: FUSC (Fully
Partial Usage of Subchannels), PUSC (Partial Usage of SubChannels) and
AMC (Adaptive modulation and coding). In
FUSC, subchannels are made of subcarriers spread over all the
frequency band. This mode is generally more adapted to mobile
users. In AMC, the subcarriers of a subchannel are adjacent instead of
being uniformly distributed over the spectrum. AMC is more adapted to
nomadic or stationary users and generally provides higher capacity.

The grouping of subcarriers into subchannels raises the problem of the
estimation of the quality of a subchannel. Theoretically channel
quality should be evaluated on each subcarrier of the corresponding
subchannel to compute the associated capacity. This work assumes that
it is possible to consider a single channel gain for all the
subcarriers making part of a subchannel (for example via channel gains
evaluated on pilot subcarriers).

We consider a circular cell $C$ of radius $R$ with a base station (BS
for short) at its center. The transmission power dedicated to each
subchannel by the base station is denoted by $P$. Each subchannel has
a total bandwidth $W$ (in kHz).  The received signal power for a
mobile station at distance $d$ from the BS can be expressed as
\begin{equation}\label{eq_dfmv_valuetools:7}
  P(d)= \frac{PK_\gamma}{d^{\gamma}}GF:=P_\gamma G d^{-\gamma},
\end{equation}
where $K_\gamma$ is a constant equal to the attenuation at a reference
distance, denoted by $\dref$, that separates far field from near field
propagation. Namely,
\begin{equation*}
  K_\gamma=\left(\frac{c}{4\pi f \dref}\right)^2\dref^\gamma,
\end{equation*}
where $f$ is the radio-wave frequency.  The variable $\gamma$ is the
path-loss exponent which indicates the power at which the path loss
increases with distance. Its depends on the specific propagation
environment, in urban area, it is in the range from 3 to 5.  It must
be noted that this propagation model is an approximate model,
difficult to calibrate for real life situations.  In particular, it
might be reasonable to envision models where $\gamma$ depends on the
distance so that the attenuation would be proportional to
$d^{\gamma(d)}$. Because of the complexity of such a model, $\gamma$
is often considered as constant but the path-loss is multiplied by two
random variables $G$ and $F$ which represent respectively the
shadowing, i.e. the attenuation due to obstacles, and the Rayleigh
fading, i.e. the attenuation due to local movements of the mobile.
Usually, $G$ is taken as a log-normal distribution: $G = 10^{S/10}$,
where $S\sim\mathcal{N}(\kappa,\,v^2)$. As to $F$, it is customary to
choose an exponential distribution with parameter $1$. Both, the
shadowing and the fading experienced by each user are supposed to be
independent from other users' shadowing and fading. For the sake of
simplicity, we will here treat the situation where only shadowing is
taken into account, the computations would be pretty much like the
forthcoming ones and the results rather similar should we consider
Rayleigh fading.

All active users in the cell compete to have access to some of the
$\Navail$ available subchannels. There are $K$ classes of users
distinguished by the transmission rate they require: $C_k$ is the rate
of class $k$ customers and $\tau_k$ denotes the probability that a
customer belongs to class $k$.  A user, at distance $d$ from the BS,
is able to receive the signal only if the
signal-to-interference-plus-noise ratio $\SNR=\frac{P(d)}{I}$ is above
some constant $\beta_{min}$ where $I$ is the noise plus interference
power and $P(d)$ is the received signal power at distance $d$, see
\eqref{eq_dfmv_valuetools:7}. If the $\SNR$ is below the critical
threshold, then the user is said to be in outage and cannot proceed
with his communication.

To avoid excess demands, the operator may impose a maximum number
$\Nmax$ of allocated subchannels to each user at each time
slot. According to the Shannon formula, for a user demanding a service
of bit rate $C_k$, located at distance $d$ from the BS and
experiencing a shadowing $g$, the number of requires subchannels is
thus the minimum of $\Nmax$ and of
\begin{equation*}
  N_{\text{user}}=
  \begin{cases}
    \left\lceil\dfrac{C_k}{W\log_2 \left(1+P_\gamma g d^{-\gamma}/I\right)}\right\rceil & \text{ if } P_\gamma g d^{-\gamma}/I\geq\beta_{min},\\
    0 & \text{ otherwise,}
  \end{cases}
\end{equation*}
where $\left\lceil x\right\rceil$ means the minimum integer number not
smaller than $x$.

We make the simplifying assumption that the allocation is made at
every time slot and that there is no buffering neither in the access
point nor in each mobile station. All the users have independently
from others a probability $p$ to have a packet to transmit at each
slot. This means, that each user has a traffic pattern which follows a
geometric process of intensity $p$. We also assume that users are
dispatched in the cell according to a Poisson process of intensity
$\lambda_0$. According to the thinning theorem for Poisson processes,
this induces that active users form a Poisson process of intensity
$\lambda=\lambda_0p$. This intensity is kept fixed over the time. That
may result from two hypothesis: Either we consider that for a small
time scale, users do not move significantly and thus the configuration
does not evolve. Alternatively, we may consider that statistically,
the whole configuration of active users has reached its equilibrium so
that the distribution of active users does not vary through time
though each user may move.

From the previous considerations, a user is characterized by three
independent parameters: his position, his class and the intensity of
the shadowing he is experiencing. We model this as a Poisson process
on $E=B(0,\, R)\times \{1,\,\cdots,\, K\} \times \R^+$ of intensity
measure
\begin{equation*}
  \lambda \d\nu(x):=\lambda (\d x \otimes \d \tau(k)\otimes \d \rho(g))
\end{equation*}
where $B(0,\, R)=\{x\in \R^2, \, \|x\|\le R\}$, $\tau$ is the
probability distribution of classes given by $\tau(\{k\})=\tau_k$ and
$\rho$ is the distribution of the random variable $G$ defined above.
We set
\begin{multline*}
  f(x,\, k,\, g)
  =\min\left(\Nmax,\right.\\
  \left.  \car_{\left\{P_\gamma g \|x\|^{-\gamma}\geq
        I\beta_{min}\right\}}\left\lceil\frac{C_k}{W\log_2
        \left(1+P_\gamma g
          \|x\|^{-\gamma}/I\right)}\right\rceil\right).
\end{multline*}
With the notations of Section \ref{sec:edgeworth-expansion},
\begin{equation*}
  \Ntot
  =\int_{\text{cell}}
  f(x,\, k,\, g) \  \d\omega(x,\, k,\, g).
\end{equation*}
We are interested in the loss probability which is given by
\begin{equation*}
  \P(\Ntot\ge \Navail).
\end{equation*}
We first need to compute the different moment of $f$ with respect to
$\nu$ in order to apply Theorem~\ref{thm_dfmv_valuetools:1} and
Theorem~\ref{thm_dfmv_valuetools:2}. For, we set
\begin{equation*}
  l_k=\Nmax \ \wedge  \left\lceil\frac{C_k}{W\log_2
      (1+\beta_{\min})}\right\rceil,
\end{equation*}
where $a\wedge b=\min(a,\, b)$. Furthermore, we introduce $\beta_{k,\,
  0}=\infty$,
\begin{equation*}
  \beta_{k,\, l}=\frac{I}{P}\left(2^{C_k/Wl}-1\right), \, 1\le k\le K,\ 1\le l\le l_k-1,
\end{equation*}
and $\beta_{k,\, l_k}=I\beta_{\min}/P.$

By the very definition of the ceiling function, we have
\begin{multline*}
  \int_E f^p\d\nu\\= \sum_{k=1}^K \tau_k\sum_{l=1}^{l_k}l^p\
  \int_{\text{cell}}\int_\R \car_{[\beta_{k,\, l};\, \beta_{k,\,
      l-1})}(g\|x\|^{-\gamma})\d\rho(g)\d x.
\end{multline*}
According to the change of variable formula, we have
\begin{multline*}
  \int_{\text{cell}} \car_{[\beta_{k,\, l};\, \beta_{k,\, l-1})}(g\|x\|^{-\gamma})\d x\\
  =\pi(\beta_{k,\, l}^{-2/\gamma}\wedge R^2 -\beta_{k,\,
    l-1}^{-2/\gamma}\wedge R^2)g^{2/\gamma}.
\end{multline*}
Thus, we have
\begin{multline*}
  \int_{\text{cell}}\int_\R \car_{[\beta_{k,\, l};\, \beta_{k,\, l-1})}(g\|x\|^{-\gamma})\d\rho(g)\d x\\
  \begin{aligned}
    &=\pi(\beta_{k,\, l}^{-2/\gamma}\wedge R^2-\beta_{k,\, l-1}^{-2/\gamma}\wedge R^2)\esp{}{10^{S/5\gamma}}\\
    &=\pi(\beta_{k,\, l}^{-2/\gamma}\wedge R^2-\beta_{k,\,
      l-1}^{-2/\gamma}\wedge R^2)\ 10^{(\kappa+\frac{v^2}{10\gamma}\ln
      10)/5\gamma}:=\zeta_{k,\, l}.
  \end{aligned}
\end{multline*}
We thus have proved the following theorem.
\begin{theorem}
  For any $p\ge 0$, with the same notations as above, we have:
  \begin{equation}
    \label{eq_dfmv_valuetools:4}
    \int f^p\d\nu=\sum_{k=1}^K \tau_k\sum_{l=1}^{l_k}l^p\, \zeta_{k,\, l}.
  \end{equation}
\end{theorem}

\section{Loss probability}
\label{sec:loss-probability}

\subsection{Exact method}
\label{sec:exact-methode}
Since $f$ is deterministic, $\Ntot$ follows a compound Poisson
distribution: it is distributed as
\begin{equation*}
  \sum_{k=1}^K\sum_{l=1}^{l_k} l \, N_{k,\, l}
\end{equation*}
where $(N_{k,\, l},\, 1\le k\le K, \, 1\le l\le l_k)$ are independent
Poisson random variables, the parameter of $N_{k,\, l}$ is $\lambda
\tau_k\zeta_{k,\, l}.$ Using the properties of Poisson random
variables, we can reduce the complexity of this expression. Let
$L=\max(l_k, \, 1\le k\le K)$ and for $l\in \{1,\,\cdots,\, L\}$, let
$K_l=\{k,\, l_k\ge l\}$. Then, $\Ntot$ is distributed as
\begin{equation*}
  \sum_{l=1}^{L} l \, M_l
\end{equation*}
where $(M_l, \, 1\le l\le l_k)$ are independent Poisson random
variables, the parameter of $M_l$ being $m_l:=\sum_{k\in
  K_l}\lambda\tau_k\zeta_{k,\, l}.$ For each $l$, it is easy to
construct an array which represents the distribution of $lM_l$ by the
following rule:
\begin{equation*}
  p_{l}(w)=
  \begin{cases}
    0 & \text{ if } w \mod l \not = 0,\\
    \exp(-m_l)m_l^q/q! &\text{ if } w=ql.
  \end{cases}
\end{equation*}
By discrete convolution, the distribution of $\Ntot$ and then its
cumulative distribution function, are easily calculable. The value of
$\Navail$ which ensures a loss probability below the desired threshold
is found by inspection. The only difficulty with this approach is to
determine where to truncate the Poisson distribution functions for
machine representation. According to large deviation theory
\cite{MR2571413},
\begin{equation*}
  \P(\text{Poisson}(\theta)\ge a \theta)\le \exp(-\theta(a\ln a +1-a)).
\end{equation*}
When $\theta$ is known, it is straightforward to choose $a(\theta)$ so
that the right-hand-side of the previous equation is smaller than the
desired threshold. The total memory size is thus proportional to
$\max(m_la(m_l)l,\, 1\le l\le l_k)$. This may be memory (and time)
consuming if the parameters of some Poisson random variables or the
threshold are small.  This method is well suited to estimate loss
probability since it gives exact results within a reasonable amount of
time but it is less useful for dimensioning purpose. Given $\Navail$,
if we seek for the value of $\lambda$ which guarantees a loss
probability less than the desired threshold, there is no better way
than trial and error. At least, the subsequent methods even imprecise
may help to evaluate the order of magnitude of $\lambda$ for the first
trial.

\subsection{Approximations}
\label{sec:approximations}
We begin by the classical Gaussian approximation. It is clear that
\begin{multline*}
  \P(\int_E f\d\omega\ge \Navail)=\P(\int_E f_\sigma (\d\omega-\lambda\d\nu)\ge N_\sigma)\\
  =\esp{\lambda\nu}{\car_{[N_\sigma,\, +\infty)}(\int_E f_\sigma
    (\d\omega-\lambda\d\nu))}
\end{multline*}
where $N_\sigma=(\Navail-\int f\lambda \d\nu)/\sigma$. Since the
indicator function $\car_{[N_\sigma,\, +\infty)}$ is not Lipschitz, we
can not apply the bound given by
Theorem~\ref{thm_dfmv_valuetools:1}. However, we can upper-bound the
indicator by a continuous function whose Lipschitz norm is not greater
than $1$. For instance, taking
\begin{equation*}
  \phi(x)=\min(x^+, 1) \text{ and } \phi_N(x)=\phi(x-N),
\end{equation*}
we have
\begin{equation*}
  \car_{[N_\sigma+1,\, +\infty)}\le \phi_{N_\sigma+1}\le  \car_{[N_\sigma,\, +\infty)}\le \phi_{N_\sigma-1}\le \car_{[N_\sigma -1,\, +\infty)}.
\end{equation*}
Hence,
\begin{multline}
  \label{eq:2}
  1-Q(N_\sigma+1)-\frac{1}{2}\sqrt{\frac2\pi}\frac{m(3,\, 1)}{\sqrt{\lambda}}\\
  \le    \P(\int_E f\d\omega\ge \Navail)\le \\
  1-Q(N_\sigma-1)+\frac{1}{2}\sqrt{\frac2\pi}\frac{m(3,\,
    1)}{\sqrt{\lambda}},
\end{multline}
where $Q$ is the cumulative distribution function of a standard
Gaussian random variable.

According to Theorem \ref{thm_dfmv_valuetools:2}, one can proceed with
a more accurate approximation. Via polynomial interpolation, it is
easy to construct a ${\mathcal C}^3$ function $\psi_N^l$ such that
\begin{equation*}
  \|(\psi_N^l)^{(3)}\|_\infty \le 1\text{ and } \car_{[N_\sigma+3.5,\, +\infty)}\le \psi_{N_\sigma}^l\le  \car_{[N_\sigma,\, +\infty)}
\end{equation*} and a function $\psi_N^r$ such that 
\begin{equation*}
  \|(\psi_N^r)^{(3)}\|_\infty \le 1\text{ and }  \car_{[N_\sigma,\, +\infty)} \le \psi_{N_\sigma}^r\le \car_{[N_\sigma-3.5,\, +\infty)}
\end{equation*}
From \eqref{eq:7}, it follows that
\begin{multline}
  \label{eq:3}
  1-Q(N_\sigma+3.5)-\frac{m(3,\, 1)}{6\sqrt{\lambda}}Q^{(3)}(N_\sigma+3.5)-E_\lambda\\
  \le \P(\int_E f\d\omega\ge \Navail)\le \\
  1-Q(N_\sigma-3.5)+\frac{m(3,\,
    1)}{6\sqrt{\lambda}}Q^{(3)}(N_\sigma-3.5)+E_\lambda
\end{multline}
where $E_\lambda$ is the right-hand-side of \eqref{eq:4} with
$\|F^{(3)}\|_\infty=1$.

Going again one step further, following the same lines, according to
\eqref{eq:edgeworth2}, one can show that
\begin{multline}
  \label{eq:6}
  \P(\int_E f\d\omega\ge \Navail)\le 1-Q(N_\sigma-6.5)\\+\frac{m(3,\,
    1)}{6\sqrt{\lambda}}Q^{(3)}(N_\sigma-6.5) +\frac{m(3,1)^2}{72
    \lambda}Q^{(5)}(N_\sigma-6.5)\\+\frac{m(4,1)}{24\lambda}Q^{(3)}(N_\sigma-6.5)+F_\lambda
\end{multline}
where $F_\lambda$ is bounded above in \eqref{eq_dfmv_valuetools:8}.

For all the approximations given above, for a fixed value of
$\Navail$, an approximate value of $\lambda$ can be obtained by
solving numerically an equation in $\sqrt{\lambda}$.
\subsection{Robust upper-bound}
\label{sec:robust-upper-bound}
If we seek for robustness and not precision, it may be interesting to
consider the so-called concentration inequality. We remark that in the
present context, $f$ is non-negative and bounded by $L=\max_k l_k$ so
that we are in position to apply Theorem
\ref{thm_dfmv_valuetools:3}. We obtain that
\begin{multline}\label{eq:5}
  \P(\int_E f\d\omega\ge \int_E f\d\nu + a)\\
  \le \exp\left(-\frac{\int_E f^2\lambda\d\nu}{L^2}g(\frac{aL}{\int_E
      f^2\lambda\d\nu})\right),
\end{multline}
where $g$ is defined in Section~\ref{sec:conc-ineq}.

\section{Applications to OFDMA and LTE}
\label{sec:appl-ofdma-lte}
In such systems, there is a huge number of physical parameters with a
wide range of variations, it is thus rather hard to explore the while
variety of sensible scenarios. For illustration purposes, we chose a
circular cell of radius $R=300$ meters equipped with an isotropic
antenna such that the transmitted power is $1$ W and the reference
distance is $10$ meters. The mean number of active customers per unit
of surface, denoted by $\lambda$, was chosen to vary between $0,001$
and $0.000\, 1$, this corresponds to an average number of active
customers varying from $3$ to $30$, a realistic value for the systems
under consideration. The minimum SINR is $0.3$~dB and the random
variable $S$ defined above is a centered Gaussian with variance equal
to $10$. There are two classes of customers, $C_1=1,000$ kb/s and
$C_2=400$ kb/s. It must be noted that our set of parameters is not
universal but for the different scenarios we tested, the numerical
facts we want to point out were always apparent.
Since the time scale is of the order of a packet transmission time,
the traffic is defined as the mean number of required subchannels at
each slot provided that the time unit is the slot duration, that is to
say that the load is defined as $\rho=\lambda\int_{\text{cell}}f\d
\nu$.

\begin{figure}[!ht]
  \centering
  \includegraphics[width=\columnwidth]{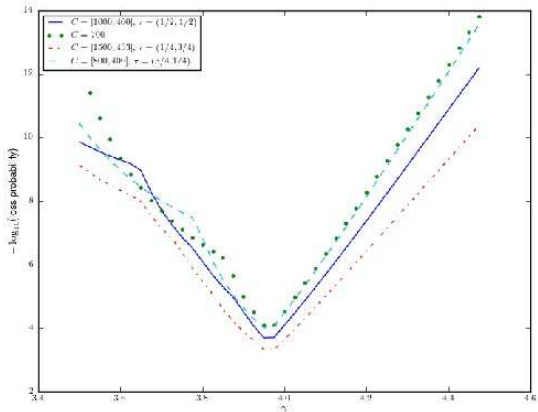}
  \caption{Impact of $\gamma$ and $\tau$ on the loss probability
    ($\Navail=92$, $\lambda=0.0001$)}
  \label{fig_dfmv_valuetools:gamma_tau}
\end{figure}

Figure \ref{fig_dfmv_valuetools:gamma_tau} shows, the loss probability
may vary up to two orders of magnitude when the rate and the
probability of each class change even if the mean rate $\sum_k \tau_k
C_k$ remains constant. Thus mean rate is not a sufficient parameter to
predict the performances of such a system.  The load $\rho$ is neither
a pertinent indicator as the computations show that the loads of the
various scenarios differs from less than $3\%$.

Comparatively, Figure \ref{fig_dfmv_valuetools:gamma_tau} shows that
variations of $\gamma$ have tremendous effects on the loss
probability: a change of a few percents of the value of $\gamma$
induces a variation of several order of magnitude for the loss
probability.  It is not surprising that the loss probability increases
as a function of $\gamma$: as $\gamma$ increases, the radio
propagation conditions worsen and for a given transmission rate, the
number of necessary subchannels increases, generating
overloading. Beyond a certain value of $\gamma$ (apparently around
3.95 on Figure \ref{fig_dfmv_valuetools:gamma_tau}), the radio
conditions are so harsh that a major part of the customers are in
outage since they do not satisfy the $\SNR$ criterion any longer.  We
remark here that the critical value of $\gamma$ is almost the same for
all configurations of classes. Indeed, the critical value $\gamma_c$
of $\gamma$ can be found by a simple reasoning: When
$\gamma<\gamma_c$, a class $k$ customer uses less than the allowed
$l_k$ subchannels because the radio conditions are good enough for
$\beta_{k,\, j}^{1/\gamma}\ge R$ for some $j<l_k$ so that the load
increases with $\gamma$. For $\gamma>\gamma_c$, all the $\beta_{k,\
  l}^{-1/\gamma}$ are lower than $R$ and the larger $\gamma$, the
wider the gap. Hence the number of customers in outage increases as
$\gamma$ increases and the load decreases. Thus,
\begin{equation*}
  \gamma_c \simeq \inf\{\gamma, \beta_{s,\, l_s-1}^{-1/\gamma}\le R\} \text{ for } s=\text{arg max}_k l_k.
\end{equation*}
If we proceed this way for the data of Figure
\ref{fig_dfmv_valuetools:gamma_tau}, we retrieve $\gamma_c=3.95$. This
means that for a conservative dimensioning, in the absence of estimate
of $\gamma$, computations may be done with this value of $\gamma$.

For a threshold given by $\epsilon=10^{-4}$, we want to find $\Navail$
such that $\P(\Ntot\ge \Navail)\le \epsilon$. As said earlier, the
exact method gives the result at the price of a sometimes lengthy
process. In view of \ref{eq:2}, one could also search for $\alpha$
such that
\begin{equation}\label{eq_dfmv_valuetools:5}
  1-Q(\alpha)+\frac12\sqrt{\frac 2\pi}m(3,\lambda)=\epsilon
\end{equation}
and then consider $\lceil1+\int_E f\d\nu +\alpha \sigma\rceil$ as an
approximate value of $\Navail$. Unfortunately and as was expected
since the Gaussian approximation is likely to be valid for
\textit{large} values of $\lambda$, the corrective term in
\eqref{eq_dfmv_valuetools:5} is far too large (between $30$ and $500$
depending on $\gamma$) for \eqref{eq_dfmv_valuetools:5} to have a
meaning. Hence, we must proceed as usual and find $\alpha$ such that
$1-Q(\alpha)=\epsilon$, i.e. $\alpha \simeq 3.71$. The approximate
value of $\Navail$ is thus given by $\lceil\int_E f\d\nu +3.71
\sigma\rceil$. The consequence is that we do not have any longer any
guarantee on the quality of this approximation, how close it is to the
true value and even more basic, whether it is greater or lower than
the correct value. In fact, it is absolutely impossible to choose a
dimensioning value lower than the true value since there is no longer
a guarantee that the loss probability is lower than~$\epsilon$. As
shows Figure \ref{fig_dfmv_valuetools:dimension}, it turns out that
the values returned by the Gaussian method are always under the true
value. Thus this annihilates any possibility to use the Gaussian
approximation for dimensioning purposes.

Going one step further, according to \eqref{eq:3}, one may find
$\alpha$ such that
\begin{equation*}
  1-Q(\alpha)-\frac{m(3,\, \lambda)}{6}Q^{(3)}(\alpha)+E_\lambda=\epsilon
\end{equation*}
and then use
\begin{equation*}
  \lceil{3.5+\int_E f\d\nu +\alpha \sigma}\rceil
\end{equation*}
as an approximate \textit{guaranteed} value of $\Navail$. By
\textit{guaranteed}, we mean that according to \eqref{eq:3}, it holds
for sure that the loss probability with this value of $\Navail$ is
smaller than $\epsilon$ even if there is an approximation process
during its computation. Since the error in the Edgeworth approximation
is of the order of $1/\lambda$, instead of $1/\sqrt{\lambda}$ for the
Gaussian approximation, one may hope that this method will be
efficient for smaller values of $\lambda$. It turns out that for the
data sets we examined, $E_\lambda$ is of the order of
$10^{-7}/\lambda$, thus this method can be used as long as
$10^{-7}/\lambda \ll \epsilon$. Otherwise, as for the Gaussian case,
we are reduced to find $\alpha$ such that
\begin{equation*}
  1-Q(\alpha)-\frac{m(3,\, \lambda)}{6}Q^{(3)}(\alpha)=\epsilon
\end{equation*}
and consider $\lceil{3.5+\int_E f\d\nu +\alpha \sigma}\rceil$ but we
no longer have any guarantee on the validity of the value. As Figure
\ref{fig_dfmv_valuetools:dimension} shows, for the considered data
set, Edgeworth methods leads to an optimistic value which is once
again absolutely not acceptable. One can pursue the development as in
\eqref{eq:edgeworth2} and use \eqref{eq:6}, thus we have to solve
\begin{multline*}
  1-Q(\alpha)-\frac{m(3,\, \lambda)}{6}Q^{(3)}(\alpha)\\
  -\frac{m(3,1)^2}{72
    \lambda}Q^{(5)}(\alpha)+\frac{m(4,1)}{24\lambda}Q^{(3)}(\alpha)
  -F_\lambda=\epsilon.
\end{multline*}
For the analog of \ref{eq:3} to hold, we have to find $\Psi$ a
$\mathcal C^5_b$ function greater than $\car_{[x,\,\infty)}$ but
smaller than $\car_{[x-\text{lag},\, \infty)}$ with a fifth derivative
smaller than $1$. Looking for $\Psi$ in the set of polynomial functions, we can find such
a function only if $\text{lag}$ is  greater than $6.5$ (for smaller value of the lag, the fifth derivative is not bounded by $1$) thus the
dimensioning value has to be chosen as:
\begin{equation*}
  \lceil{6.5+\int_E f\d\nu +\alpha \sigma}\rceil.
\end{equation*}
For the values we have, it turns out that $F_\lambda$ is of the order
of $10^{-9}\lambda^{-3/2}$ which is negligible compared to
$\epsilon=10^{-4}$, so that we can effectively use this method for $\lambda\ge 10^{-4}$. As it
is shown in Figure~\ref{fig_dfmv_valuetools:dimension}, the values
obtained with this development are very close to the true values but
always greater as it is necessary for the guarantee. The procedure
should thus be the following: compute the error bounds given by
\eqref{eq:2}, \eqref{eq:4} and \eqref{eq:6} and find the one which gives
a value negligible with respect to the threshold $\epsilon$,
then use the corresponding dimensioning formula. If none is suitable,
use a finer Edgeworth expansion or resort to the concentration inequality approach.

Note that the Edgeworth method requires the computations of the first
three (or five) moments, whose lengthiest part is to compute the
$\zeta_{k,\, l}$ which is also a step required by the exact
method. Thus Edgeworth methods are dramatically simpler than the exact
method and may be as precise. However, both the exact and Edgeworth
methods suffer from the same flaw: There are precise as long as the
parameters, mainly $\lambda$ and $\gamma$, are perfectly well
estimated. The value of $\gamma$ is often set empirically (to say the
least) so that it seems important to have dimensioning values robust
to some estimate errors. This is the goal of the last method we
propose.

According to \eqref{eq:5}, if we find $\alpha$ such that
\begin{equation*}
  g(\frac{\alpha L}{\int_E f^2\lambda\d\nu})=-\frac{\log(\epsilon)L^2}{\int_E f^2\lambda \d\nu}
\end{equation*}
and
\begin{equation}\label{eq_dfmv_valuetools:6}
  \Navail=\int_E f\d\nu+\frac{\alpha}{L^2}\int_E f^2\lambda \d\nu,
\end{equation}
we are sure that the loss probability will fall under
$\epsilon$. However, we do not know a priori how larger this value of
$\Navail$ than the true value. It turns out that the relative
oversizing increases with $\gamma$ from a few percents to $40\%$ for
the large value of $\gamma$ and hence small values of $\Navail$.  For
instance, for $\gamma=4.2$, the value of $\Navail$ given by
\eqref{eq_dfmv_valuetools:6} is $40$ whereas the exact value is $32$
hence an oversizing of $25\%$. However, for $\gamma=4.12$, which is
$2\%$ away from $4.2$, the required number of subchannels is also
$40$. The oversizing is thus not as bad as it may seem since it may be
viewed as a protection against traffic increase, epistemic risk (model
error) and estimate error.

\begin{figure}[!ht]
  \centering
  \includegraphics[width=\columnwidth]{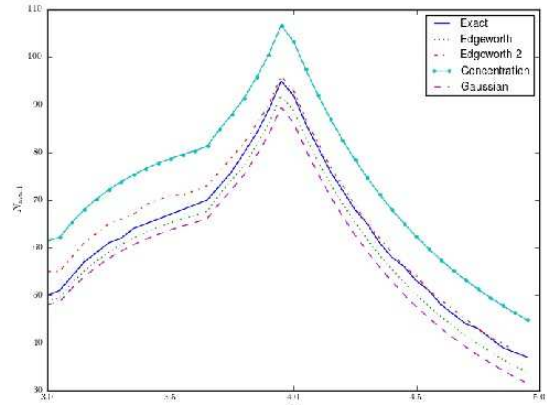}
  \caption{Estimates of $\Navail$ as a function of $\gamma$ by the
    different methods}
  \label{fig_dfmv_valuetools:dimension}
\end{figure}


\appendix

\section{Hermite polynomials}
\label{sec:hermite-polynomials}
Let $\Phi$ be the Gaussian probability density function:
$\Phi(x)=\exp(-x^2/2)/\sqrt{2\pi}$ and $\mu$ the Gaussian measure on $\R$.  Hermite polynomials $(H_k,\, k\ge
0)$ are defined by the recursion formula:
\begin{equation*}
  H_k(x)\Phi(x)=\frac{d^k}{dx^k}\Phi(x).
\end{equation*}
For the sake of completness, we recall that
\begin{multline*}
  H_0(x)=1,\, H_1(x)=x,\, H_2(x)=x^2-1,\, H_3(x)=x^3-3x\\
  H_4(x)=x^4-6x^2+3,\ H_5(x)=x^5-10x^3+15x.
\end{multline*}
Thus, for $F\in {\mathcal C}^k_b$, using integration by parts, we have
\begin{equation}\label{eq_dfmv_valuetools:9}
  \int_\R F^{(k)}(x)\d\mu(x)= \int_\R F(x)H_k(x)\d\mu(x).
\end{equation}
Let $Q(x)=\int_{-\infty}^x \Phi(u)\d u=\int_\R \car_{(-\infty;\,
  x]}(u)\Phi(u)\d u$. Then, $Q^\prime=\Phi$ and
\begin{multline}\label{eq:7}
  \int_\R \car_{(-\infty;\, x]}(u)H_k(u)\d\mu(u) \\=\int_\R
  \car_{(-\infty;\, x]}(u)\frac{d^{k+1}}{dx^{k+1}}Q(u)\d u\\
  =Q^{(k)}(x)= H_{k-1}(x)\Phi(x).
\end{multline}

\section{Edgeworth expansion}
\label{sec:edgeworth-expansion}
For details on Poisson processes, we refer to
\cite{FnT1,Decreusefond:2012sys}. For $E$ a Polish space equipped with
a Radon measure $\nu$, $\Gamma_E$ denotes the set of locally finite
discrete measures on $E$. The generic element $\omega$ of $\Gamma_E$
may be identified with a set $\omega=\{x_n,\, n\ge 1\}$ such that
$\omega \cap K$ has finite cardinal for any $K$ compact in $E$. We
denote by $\int_E f\d\omega$ the sum $\sum_{x\in \omega} f(x)$
provided that it exists as an element of $\R\cup\{+\infty\}$. A
Poisson process of intensity $\nu$ is a probability $\P_\nu$ on $\Gamma_E$, such that for any
$f\in {\mathcal C}_K(E, \; \R)$,
\begin{equation*}
  \esp{\nu}{\exp(-\int_E f\d\omega)}=\exp(-\int_E 1-e^{-f(x)}\d\nu(x)).
\end{equation*}
For $f\in L^1(\nu)$, the Campbell formula states that
\begin{equation*}
  \esp{\nu}{\int f\d \omega}=\int f\d\nu.
\end{equation*}
We introduce the discrete gradient $D$ defined by
\begin{equation*}
  D_xF(\omega)=F(\omega\cup\{x\})-F(\omega), \text{ for all } x\in E.
\end{equation*}
In particular, for $f\in L^1(\nu)$, we have
\begin{equation*}
  D_x\int_Ef\d\omega=f(x).
\end{equation*}
The domain of $D$, denoted by $\dom D$ is the set of functionals $F\,
:\, \Gamma_E\to R$ such that
\begin{equation*}
  \esp{\nu}{\int_E |D_xF(\omega)|^2\d\nu(x)}<\infty.
\end{equation*}
The integration by parts then says that, for any $F\in \dom D$, any
$u\in L^2(\nu)$,
\begin{multline}\label{eq_dfmv_valuetools:1}
  \esp{\nu}{F\int_E u(x) (\d\omega(x)-\d\nu(x))}\\
  =\esp{\nu}{\int_ED_xF \, u(x)\d\nu(x)}.
\end{multline}
We denote by $\sigma=\|f\|_{L^2(\nu)}\sqrt{\lambda}$ and
$f_\sigma=f/\sigma$. Note that $\|f_\sigma\|_{L^2(\nu)}=1/\lambda$ and
that
\begin{equation*}
  m(p,\,\lambda):=\int_E |f_\sigma(x)|^p \lambda\d\nu(x)=\|f\|_{L^2(\nu)}^{-p}\|f\|^p_{L^p(\nu)}\lambda^{1-p/2}.
\end{equation*}
The proof of the following theorem may be found in
\cite{Decreusefond:2012sys,taqqu,MR2531026}.
\begin{theorem}
  \label{thm_dfmv_valuetools:1}
  Let $f\in L^2(\nu)$. For $\lambda >0$, let
  \begin{equation*}
    N^\lambda=\int_E f_\sigma(x)(\d\omega(x)-\lambda\d\nu(x)).
  \end{equation*}
  Then, for any Lipschitz function $F$ from $\R$ to $\R$, we have
  \begin{equation*}
    \left| \esp{\lambda\nu}{F(N^\lambda)}-\int_\R F\d\mu\right|
    \le \frac12\sqrt{\frac\pi 2}\ m(3,\, \lambda)\,\|F\|_{\lip}.
  \end{equation*}
\end{theorem}
To prove the Edgeworth expansion and its error bound, we introduce
some notions of Gaussian calculus. For $F\in {\mathcal
  C}_b^2(\R;\,\R)$, we consider
\begin{equation*}
  AF(x)=xF^\prime(x)-F^{\prime\prime}(x), \text{ for any } x\in \R.
\end{equation*}
The Ornstein-Uhlenbeck semi-group is defined by
\begin{equation*}
  P_tF(x)=\int_\R F(e^{-t} x+\sqrt{1-e^{-2t}}y)\d\mu(y) \text{ for any } t\ge 0.
\end{equation*}
The
infinitesimal generator $A$ and $P_t$ are linked by the following identity
\begin{equation}\label{eq_dfmv_valuetools:3}
  F(x)-\int_\R F(y)\d\mu(y)=-\int_0^\infty AP_tF(x)\d t.
\end{equation}

\begin{theorem}
  \label{thm_dfmv_valuetools:2}
  For $F\in {\mathcal C}^3_b(\R,\, \R)$,
  \begin{multline}\label{eq:4}
    \left| \esp{\lambda\nu}{F(N^\lambda)}-\int_\R F(y)\d\mu(y)\right.\\
    \left.    -\frac16\  m(3,\, \lambda)\ \int_\R F(y)H_3(y)\d\mu(y)\vphantom{\esp{lambda}{}}\right| \\
    \le \left( \frac{m(3,\, 1)^2}{6}+\frac{m(4,\, 1)}{9}
      \sqrt{\frac{2}{\pi}}\right)\frac{\|F^{(3)}\|_\infty}{\lambda}\cdotp
  \end{multline}
\end{theorem}
\begin{proof}
  According to the Taylor formula,
  \begin{multline}
    \label{eq:1}
    D_x G(N^\lambda)=G(N^\lambda+f_\sigma(x))-G(N^\lambda)\\
    =G^\prime(N^\lambda)f_\sigma(x)
    +\frac{1}{2}f^2_\sigma(x)\, G^{\prime\prime}(N^\lambda)\\
    +\frac{1}{2}f_\sigma(x)^3\int_0^1 r^2G^{(3)}(rN^\lambda +(1-r)
    f_\sigma(x))\d r.
  \end{multline}
  Hence, according to \eqref{eq_dfmv_valuetools:1} and \eqref{eq:1},
  \begin{multline*}
    \esp{\lambda\nu}{N^\lambda (P_tF)^\prime(N^\lambda)}\\
    \begin{aligned}
      &= \esp{\lambda\nu}{\int_E f_\sigma(x) D_x (P_tF)^\prime(N^\lambda)\lambda\d\nu(x)}\\
      &=\esp{\lambda\nu}{(P_tF)^{\prime\prime}(N^\lambda)}\\
      &+ \frac{1}{2} \int_E f_\sigma^3(x)\lambda\d \nu(x)\esp{\lambda\nu}{(P_tF)^{(3)}(N^\lambda)}\\
      &+\frac{1}{2} \int_E f_\sigma^4(x)\lambda\d \nu(x)\\
      &\qquad \times \esp{\lambda\nu}{\int_0^1(P_tF)^{(4)}(rN^\lambda+(1-r)f_\sigma(x))r^2\d r}  \\
      &=A_1+A_2+A_3.
    \end{aligned}
  \end{multline*}
  It is well known that for $F\in {\mathcal C}^k$, $(x\mapsto
  P_tF(x))$ is $k+1$-times differentiable and that we have two
  expressions of the derivatives (see \cite{MR0461643}):
  \begin{multline*}
    (P_tF)^{(k+1)}(x)\\
    =\frac{e^{-(k+1)t}}{\sqrt{1-e^{-2t}}}\int_\R
    F^{(k)}(e^{-t}x+\sqrt{1-e^{-2t}}y)y\d\mu(y).
  \end{multline*}
  and $(P_tF)^{(k+1)}(x)=e^{-(k+1)t}P_tF^{(k)}(x)$. The former
  equation induces that
  \begin{multline*}
    \|(P_tF)^{(k+1)}\|_\infty\le \frac{e^{-(k+1)t}}{\sqrt{1-e^{-2t}}}\|F^{(k)}\|_\infty \int_\R|y|\d\mu(y)\\
    = \frac{e^{-(k+1)t}}{\sqrt{1-e^{-2t}}}\sqrt{\frac2\pi}\
    \|F^{(k)}\|_\infty .
  \end{multline*}
  Hence,
  \begin{equation*}
    |A_3|\le \frac{e^{-4t}}{6\sqrt{1-e^{-2t}}}\sqrt{\frac2\pi}m(4,\, \lambda)\ \|F^{(3)}\|_\infty.
  \end{equation*}
  Moreover, according to Theorem \ref{thm_dfmv_valuetools:1},
  \begin{multline*}
    \left  | \esp{\lambda\nu}{(P_tF)^{(3)}(N^\lambda)}-\int_\R (P_tF)^{(3)}(x)\d\mu(x)\right|\\
    \begin{aligned}
      &\le  \frac12\sqrt{\frac\pi 2}\,m(3,\, \lambda)\|(P_tF)^{(4)}\|_\infty\\
      &= \frac12\sqrt{\frac\pi 2} \,m(3,\, \lambda)e^{-3t}\|(P_tF^{(3)})^\prime\|_\infty \\
      &\le\frac12\, m(3,\, \lambda)\frac{e^{-4t}}{\sqrt{1-e^{-2t}}}
      \|F^{(3)}\|_\infty .
    \end{aligned}
  \end{multline*}
  Then, we have,
  \begin{multline*}
    |A_2-\frac12{m(3,\, \lambda)} \int_\R (P_tF)^{(3)}(x)\d\mu(x)|\\
    \le \frac14\, m(3,\, \lambda)^2\frac{e^{-4t}}{\sqrt{1-e^{-2t}}}
    \|F^{(3)}\|_\infty .
  \end{multline*}
  Hence,
  \begin{multline*}
    \esp{\lambda\nu}{N^\lambda (P_tF)^\prime(N^\lambda)-(P_tF)^{\prime\prime}(N^\lambda)}\\
    =\frac12m(3,\, \lambda)\int_\R(P_tF)^{(3)}(x)\d\mu(x)+R(t),
  \end{multline*}
  where
  \begin{equation*}
    R(t)\le\left( \frac{m(3,\, \lambda)^2}{4}+\frac{m(4,\, \lambda)}{6} \sqrt{\frac{2}{\pi}}\right)\|F^{(3)}\|_\infty\frac{e^{-4t}}{\sqrt{1-e^{-2t}}}\cdotp
  \end{equation*}
  Now then,
  \begin{multline*}
    \int_\R(P_tF)^{(3)}(x)\d\mu(x)\\
    \begin{aligned}
&    = e^{-3t} \int_R\int_\R F^{(3)}(e^{-t}x+\sqrt{1-e^{-2t}}y)\d\mu(y)\\
&    =e^{-3t}\int_\R F^{(3)}(y)\d\mu(y)\\ &=e^{-3t}\int_R F(y)H_3(y)\d\mu(y),
    \end{aligned}
  \end{multline*}
  since the Gaussian measure on $\R^2$ is rotation invariant and
  according to \eqref{eq_dfmv_valuetools:9}.  Remarking
  that $$\int_0^\infty e^{-4t}(1-e^{-2t})^{-1/2}\d t =2/3$$ and
  applying \eqref{eq_dfmv_valuetools:3} to $x=N^\lambda$, the result
  follows.
\end{proof}
This development is not new in itself but to the best of our
knowledge, it is the first time that there is an estimate of the error
bound. Following the same lines, we can pursue the expansion up to any
order provided that $F$ be sufficiently differentiable. Namely, for
$F\in{\mathcal C}^5_b$, we have
\begin{multline}\label{eq:edgeworth2}
  \esp{\lambda\nu}{F(N^\lambda)}=\int_\R F(y)\d\mu(y)\\
  +\frac{m(3,1)}{6\sqrt{\lambda}}\int_\R F^{(3)}(y)\d\mu(y)+\frac{m(3,1)^2}{72 \lambda}\int_\R F^{(5)}(y)\d\mu(y)\\
  +\frac{m(4,\, 1)}{24\lambda}\int_\R
  F^{(4)}(y)\d\mu(y)+F_\lambda\|F^{(5)}\|_\infty.
\end{multline}
where
\begin{multline}\label{eq_dfmv_valuetools:8}
  F_\lambda
  \le \frac{m(3,1)}{\lambda^{3/2}}\left(\frac{2}{45}\, m(3,1)^2\right.\\
  \left.+(\frac{4}{135}+\frac{\pi^2}{128})\sqrt{\frac{2}\pi}\
    m(4,1)\right).
\end{multline}

\section{Concentration inequality}
\label{sec:conc-ineq}
We are now interested in an upper bound, which is called
concentration inequality.
\begin{theorem}
  \label{thm_dfmv_valuetools:3}
  Let $M,a>0$.  Assume that $\vert f(z)\vert\leq M$ $\nu-$a.s and
  $f\in L^2(E,\nu)$, then
  \begin{equation}\label{equ: Upper bound linear PPP 2}
    \P(F>\Esp{F}+a)   \leq\exp\left\{-\frac{M^2}{\Var{F}} g\left(\frac{a.M}{\Var{F}}\right)\right\}
  \end{equation}
  where $g(u)=(1+u)\ln(1+u)-u$.
\end{theorem}
The above theorem can be directly derived from
\cite{Wu:2000lr}. However let us take this opportunity to prove this
theorem in a very nice, simple and elementary fashion, exactly the
same way as Bennett built his concentration inequality for the sum of
$n$ i.i.d random variables.
\begin{proof}
  Using Chernoff's bound we have:
  \begin{align*}
    \P(F>\Esp{F}+a)    &\leq \Esp{e^{\theta F}}/e^{\theta(\Esp{F}+a)}\\
    &= e^{\int_E\left(e^{\theta f(z)}-1-\theta f(z)\right)\d
      \nu(z)-\theta a}
  \end{align*}
  Now assume that $\vert f(z)\vert\leq M$ $\nu-$a.s . Observe that the
  function $({e^x-1-x})/{x^2}$ is increasing on $\R$ (the value at
  $0$ is $1/2$), we have that
  \[e^{\theta f(z)} -\theta f(z) -1 \leq \frac{e^{\theta M}-1-\theta
    M}{M^2}f^2(z) \ \nu \text{ a.s.}\]  Thus,
  \begin{multline*}
    \P(F>\Esp{F}+a)    \\\leq \exp\left\{\int_{E}\left(\frac{e^{\theta M}-\theta M-1}{M^2}f^2(z)\right)\d \nu(z)-\theta a\right\}\\
    = \exp\left\{\frac{e^{\theta M}-1-\theta M}{M^2}\Var{F}-\theta
      a\right\}\cdot
  \end{multline*}
  We find that $\theta =  
\ln\left(1+{aM}/{\Var{F}}\right)/M$ minimizes the right-hand-side 
  and thus  we obtain \eqref{equ: Upper bound linear PPP
    2}.
\end{proof}

\balancecolumns

\end{document}